\documentclass{article}
\usepackage[paperheight=11in,paperwidth=8.5in]{geometry}
\usepackage{proceed2e}

\usepackage{times}
\usepackage{amsmath}
\usepackage{amsfonts}
\usepackage{amssymb}
\usepackage{amsthm}
\usepackage[round]{natbib}

\DeclareMathOperator*{\argmin}{arg\,min}

\newtheorem{definition}{Definition}

\newtheorem{lemma}{Lemma}
\newtheorem{theorem}{Theorem}
\newtheorem{proposition}{Proposition}
\newtheorem{corollary}{Corollary}

\title{Representation Learning for Clustering: A Statistical Framework}

\author{} 

\author{
{\bf Hassan Ashtiani} {\textnormal {and}} {\bf Shai Ben-David} \\
David R. Cheriton School of Computer Science \\
University of Waterloo,\\
Waterloo, Ontario, Canada\\
\{mhzokaei,shai\}@uwaterloo.ca \\
}

\begin{document}

\maketitle

\begin{abstract}
We address the problem of communicating domain knowledge from a user to the designer of a clustering algorithm.
We propose a protocol in which the user provides a clustering of a relatively small random sample of a data set. The algorithm designer then uses that sample to come up with a data representation under which $k$-means clustering results in a clustering (of the full data set) that is aligned with the user's clustering. We provide a formal statistical model for analyzing the sample complexity of learning a clustering representation with this paradigm.
We then introduce a notion of capacity of a class of possible representations, in the spirit of the VC-dimension, showing that classes of representations that have 
finite such dimension can be successfully learned with sample size error bounds, and end our discussion with an analysis of that dimension for classes of representations induced by linear embeddings.
\end{abstract}

\section{INTRODUCTION}

Clustering can be thought as the task of automatically dividing a set of objects into ``coherent'' subsets. This definition is not concrete, but its vagueness allows it to serve as an umbrella term for a wide diversity of algorithmic paradigms. Clustering algorithms are being routinely applied in a huge variety of fields.

Given a dataset that needs to be clustered for some application, one can choose among a variety of different clustering algorithms, along with different pre-processing techniques, that are likely to result in dramatically different answers. It is therefore critical to incorporate prior knowledge about the data and the intended semantics of the clustering into the process of picking a clustering algorithm (or, clustering model selection). Regretfully, there seem to be no systematic tool for incorporation of domain expertise for clustering model selection, and such decisions are usually being made in embarrassingly  \textit{ad hoc} ways. 
This paper aims to address that critical deficiency in a formal statistical framework. 

We approach the challenge by considering a scenario in which the domain expert (i.e., the intended user of the clustering) conveys her domain knowledge by providing a clustering of a small random subset of her data set. For example, consider a big customer service center that wishes to cluster incoming requests into groups to streamline their handling. Since the data base of requests is too large to be organized manually, the service center wishes to employ a clustering program. As the clustering designer, we would then ask the service center to pick a random sample of requests, manually cluster them, and show us the resulting grouping of that sample.  
The clustering tool then uses that sample clustering to pick a clustering method that, when applied to the full data set, will result in a clustering that follows the patterns demonstrated by that sample clustering. We address this paradigm from a statistical machine learning perspective. Aiming to achieve generalization guaranties for such an approach, it is essential to introduce some \emph{inductive bias}. We do that by restricting the clustering algorithm to a predetermined hypothesis class (or a set of concrete clustering algorithms).

 In a recent Dagstuhl workshop, \cite{blum2014Approximation} proposed to do that by fixing a clustering algorithm, say $k$-means, and searching for a metric over the data under which $k$-means optimization yields a clustering that agrees with the training sample clustering. One should note that, given any domain set $X$, for any $k$-partitioning $P$ of $X$, there exists some distance function $d_P$ over $X$ such that $P$ is the optimal $k$-means clustering solution to the input $(X, d_P)$\footnote{This property is sometimes called $k$-Richness}. Consequently, to protect against potential overfitting, the class of potential distance functions should be constrained.  In this paper, we provide (apparently the first) concrete formal framework for such a paradigm, as well as a generalization analysis of this approach. 

In this work we focus on center based clustering - an important class of clustering algorithms. In these algorithms, the goal is to find a set of ``centers'' (or prototypes), and the clusters are the Voronoi cells induced by this set of centers. The objective of such a clustering is to minimize the expected value of some monotonically increasing function of the distances of points to their cluster centers. The k--means clustering objective is arguably the most popular clustering paradigm in this class. Currently, center-based clustering tools lack a vehicle for incorporating domain expertise. Domain knowledge is usually taken into account only through an ad hoc choice of input data representation. Regretfully, it might not be realistic to require the domain expert to translate sufficiently elaborate task-relevant knowledge into hand-crafted features. 

As a model for learning representations, we assume that the user-desirable clustering can be approximated by first mapping the sample to some Euclidean (or Hilbert) space and then performing $k$-means clustering in the mapped space (or equivalently, replacing the input data metric by some kernel and performing center-based clustering with respect to that kernel). However, the clustering algorithm is supposed to learn a suitable mapping based on the given sample clustering.

The main question addressed in this work is that of the sample complexity: what is the size of a sample, to be clustered by the domain expert, that suffices for finding a close-to-optimal mapping (i.e., a mapping that generalizes well on the test data)? Intuitively, this sample complexity depends on the richness of the class of potential mappings that the algorithm is choosing from. In standard supervised learning, there are well established notions of capacity of hypothesis classes (e.g., VC-dimension) that characterize the sample complexity of learning. This paper aims to provide such relevant notions of capacity for clustering.

\subsection{Previous Work}

In practice, there are methods that use some forms of supervision for clustering. These methods are sometimes called ``semi-supervised clustering'' (\cite{basu2002semi,basu2004probabilistic, kulis2009semi}). The most common method to convey such supervision is through a set of pairwise \emph{must/cannot-link} constraints on the instances (\cite{wagstaff2001constrained}). A common way of using such information is by changing the objective of clustering so that violations of these constraints are penalized (\cite{demiriz1999semi,law2005model, basu2008constrained}). Another approach, which is closer to ours, keeps the clustering optimization objective fixed, and instead, searches for a metric that best fits given constraints. The metric is learned based on some objective function over metrics (\citep{xing2002distance, alipanahi2008distance}), so that pairs of instances marked \emph{must-link} will be close in the new metric space (and \emph{cannot-link} pairs be considered  as far apart). The two above approaches can also be integrated (\cite{bilenko2004integrating}). However, these objective functions are usually rather  ad hoc. In particular, it is not clear in what sense they are compatible with the adopted clustering algorithm (such as k-means clustering). 

A different approach to the problem of communicating user expertise for the purpose of choosing a clustering tool is discussed in \cite{ackerman2010towards}. They  considered a set of \textit{properties}, or \textit{requirements}, for clustering algorithms, and investigated which of those properties hold for various algorithms. The user can then pick the right algorithm based on the requirements that she wants the algorithm to meet. However, to turn such an approach into a practically useful tool, one will need to come up with properties that are relevant to the end user of clustering --a goal that is still far from being reached. 

Statistical convergence rates of sample clustering to the optimal clustering, with respect to some data generating probability distribution, play a central role in our analysis.
From that perspective, most relevant to our paper are results that provide generalization bounds for k-means clustering. \cite{ben2007framework} proposed the first dimension-independent generalization bound for k-means clustering based on compression techniques. \cite{biau2008performance} tightened this result by an analysis of Rademacher complexity. \cite{maurer2010dimensional} investigated a more general framework, in which generalization bounds for k-means as well as other algorithms can be obtained. It should be noted that these results are about the standard clustering setup (without any supervised feedback), where the data representation is fixed and known to the clustering algorithm.

\subsection{Contributions}

Our first contribution is to provide a statistical framework to analyze the problem of learning representation for clustering. We assume that the expert has some implicit target clustering of the dataset in his mind. The learner however, is unaware of it, and instead has to select a mapping among a set of potential mappings, under which the result of k-means clustering will be similar to the target partition. An appropriate notion of loss function is introduced to quantify the success of the learner. Then, we define the analogous notion of PAC-learnability\footnote{PAC stands for the well known notion of ``probably approximately correct'', popularized by \cite{valiant1984theory}.} for the problem of learning representation for clustering. 

The second contribution of the paper is the introduction of a combinatorial parameter, a specific notion of the capacity of the class of mappings, that determines the sample complexity of the clustering learning tasks. This combinatorial notion is a multivariate version of \emph{pseudo-dimension} of a class of real-valued mappings. We show that there is \emph{uniform convergence}  of empirical losses to the true loss, over any class of embeddings, ${\cal F}$, at a rate that is determined by the proposed dimension of that ${\cal F}$. 
This implies that any empirical risk minimization algorithm (ERM) will successfully learn such a class from sample sizes upper bounded by those rates. Finally, we analyze a particular natural class --the class of linear mappings from $\mathbb{R}^{d_2}$ to $\mathbb{R}^{d_1}$--  and show that a roughly speaking, sample size of $O(\frac{d_1d_2}{\epsilon^2})$ is sufficient to guarantee an $\epsilon$-optimal representation.

The rest of this paper is organized as follows: Section~\ref{sec:setting} defines the problem setting. Then in Section ~\ref{sec:analysis}, we investigate ERM-type algorithms and show that, ``uniform convergence'' is sufficient for them to work. Furthermore, this section presents the uniform convergence results and the proof of an upper bound for the sample complexity. Finally, we conclude in section~\ref{sec:conc} and provide some directions for future work.


\section{PROBLEM SETTING}
\label{sec:setting}

\subsection{Preliminaries}

Let $X$ be a finite domain set. A \textit{$k$-clustering} of $X$ is a partition of $X$ into $k$ subsets. If $C$ is a $k$-clustering, we denote the subsets of the partition by $C_1,...,C_k$, therefore we have $C=\{C_1,..,C_k\}$. Let $\pi^k$ denote the set of all permutations over $[k]$ where $[k]$ denotes $\{1,2,...,k\}$. The clustering difference between two clusterings, $C^1$ and $C^2$, with respect to $X$ is defined by

\begin{equation}
\Delta_{X}(C^1, C^2) = \min_{\sigma \in \pi^k } \frac{1}{|X|} \sum_{i=1}^{k} |C^1_i \Delta C^2_{\sigma(i)}|
\end{equation}

where $|.|$ and $\Delta$ denote the cardinality and the symmetric difference of sets respectively. For a sample $S\subset X$, and $C^1$ (a partition of $X$), we define $C^1\Big|_S$ to be a partition of $S$ induced by $C^1$, namely $C^1\Big|_S = \{C^1_1\cap S,\ldots,C^1_k\cap S \}$. Accordingly, the sample-based difference between two partitions is defined by

\begin{equation}
\Delta_{S}(C^1, C^2) = \Delta_{S}(C^1\Big|_S, C^2\Big|_S)
\end{equation}

Let $f$ be a mapping from $X$ to $\mathbb{R}^{d}$, and $\mu=(\mu_1, \ldots \mu_k)$ be a vector of $k$ centers in $\mathbb{R}^d$. The clustering defined by $(f, \mu)$ is the partition over $X$ induced by the $\mu$-Voronoi partition in $\mathbb{R}^d$. Namely,
\[C_f(\mu)=(C_1, \ldots C_k), ~\mbox{where for all $i$}, ~ \]\[C_i=\{x\in X: \|f(x)-\mu_i\|_2 \leq \|f(x)-\mu_j\|_2 ~\mbox{for all} ~ j\neq i \}\]

The k-means cost of clustering $X$ with a set of centers $\mu=\{\mu_1,\ldots,\mu_k\}$ and with respect to a mapping $f$ is defined by

\begin{equation}
COST_{X}(f, \mu) = \frac{1}{|X|} \sum_{x \in X} \min_{\mu_i \in \mu} \|f(x)-\mu_i\|_2^2
\end{equation}

The k-means clustering algorithm finds the set of centers $\mu^f_{X}$ that minimize this cost\footnote{We assume that the solution to k-means clustering is unique. We will elaborate about this issue in the next sections.}. In other words,

\begin{equation}
\mu^f_{X} = \argmin_{\mu} COST_{X}(f,\mu)
\end{equation}

Also, for a partition $C$ and mapping $f$, we can define the cost of clustering as follows.

\begin{equation}
COST_{X}(f, C) = \frac{1}{|X|} \sum_{i \in [k]} \min_{\mu_j} \sum_{x \in C_i} \|f(x)-\mu_j\|_2^2
\end{equation}

For a mapping $f$ as above, let $C^f_{X}$ denote the $k$-means clustering of $X$ induced by $f$, namely  

\begin{equation}
C^f_{X}=C_f(\mu^f_{X})
\end{equation}

The difference between two mappings $f_1$ and $f_2$ with respect to $X$ is defined by the difference between the result of k-means clustering using these mappings. Formally,

\begin{equation}
\Delta_{X}(f_1, f_2) = \Delta_{X}(C^{f_1}_{X},  C^{f_2}_{X})
\end{equation}

The following proposition shows the ``$k$-richness'' property of k-means objective.

\begin{proposition} Let $X$ be a domain set. For every $k$-clustering of $X$, $C$, and every $d\in \mathbb{N}^+$, there exist a mapping $g: X \mapsto \mathbb{R}^d$ such that $C^g_X = C$. 
\end{proposition}
\begin{proof} The mapping $g$ can be picked such that it collapses each cluster $C_i$ into a single point in $\mathbb{R}^n$ (and so the image of $X$ under mapping $g$ will be just $k$ single points in $\mathbb{R}^n$). The result of k-means clustering under such mapping will be $C$.
\end{proof}

In this paper, we investigate the \emph{transductive} setup, where there is a given data set, known to the learner, that needs to be clustered.
Clustering often occurs as a task over some data generating distribution
(e.g., \cite{von2005towards}). The current work can be readily extended to that setting. However, in that case, we assume that the clustering algorithm gets, on top of the clustered sample, a large unclustered sample drawn form that data generating distribution.

\subsection{Formal Problem Statement}

Let $C^*$ be the target $k$-clustering of $X$. 
A (supervised) \textit{representation learner} for clustering, takes as input a sample $S\subset X$ and its clustering, $C^*\Big|_S$, and outputs a mapping $f$ from a set of potential mappings $\mathcal{F}$. In the following, PAC stands for the notion of ``probably approximately correct''.

\begin{definition}{PAC Supervised Representation Learner for K-Means (PAC-SRLK)}

Let $\mathcal{F}$ be a set of mappings from $X$ to $\mathbb{R}^d$. A representation learning algorithm $A$ is a PAC-SRLK with sample complexity $m_{\mathcal{F}}:(0,1)^2\mapsto \mathbb{N}$ with respect to $\mathcal{F}$, if for every $(\epsilon,\delta)\in (0,1)^2$, every domain set $X$ and every clustering of $X$, $C^*$, the following holds: 

if $S$ is a randomly (uniformly) selected subset of $X$ of size at least $m_{\mathcal{F}}(\epsilon, \delta)$, then with probability at least $1-\delta$

\begin{equation}
\Delta_{X}(C^*,  C^{f_A}_X) \leq \inf_{f\in \mathcal{F}} \Delta_{X}(C^*,C^{f}_X) + \epsilon
\end{equation}

where $f_A = A(S, C^*\Big|_S)$, is the output of the algorithm.

\end{definition}

This can be regarded as a formal PAC framework to analyze the problem of learning representation for k-means clustering. The learner is compared to the best mapping in the class $\mathcal{F}$.

A natural question is providing bounds on the sample complexity of PAC-SRLK with respect to $\mathcal{F}$. Intuitively, for richer classes of mappings, we need larger clustered samples. Therefore, we need to introduce an appropriate notion of ``capacity'' for $\mathcal{F}$ and bound the sample complexity based on it. This is addressed in the next sections.

\section{ANALYSIS AND RESULTS}
\label{sec:analysis}

\subsection{Empirical Risk Minimization}

In order to prove an upper bound for the sample complexity of representation learning for clustering, we need to consider an algorithm, and prove a sample complexity bound for it. Here, we show that any ERM-type algorithm can be used for this purpose. Therefore, we will be able to prove an upper bound for the sample complexity of PAC-SRLK. 

Let $\mathcal{F}$ be a class of mappings and $X$ be the domain set. A TERM\footnote{TERM stands for Transductive Empirical Risk Minimizer} learner for $\mathcal{F}$ takes as input a sample $S\subset X$ and its clustering $Y$ and outputs:

\begin{equation}
A^{TERM}(S, Y) = \argmin_{f\in \mathcal{F}} \Delta_S(C^{f}_{X}\Big|_S, Y)
\end{equation}

Note that we call it transductive, because it is implicitly assumed that it has access to unlabeled dataset (i.e., $X$). A TERM algorithm goes over all mappings in $\mathcal{F}$ and selects the mapping which is the most consistent mapping with the given clustering: the mapping under which if we perform k-means clustering of $X$, the sample-based $\Delta$-difference between the result and $Y$ is minimized. 

Note that we are not studying this algorithm as a computational tool; we only use it to show an upper bound for the sample complexity.



Intuitively, this algorithm will work well when the empirical $\Delta$-difference and the true $\Delta$-difference of the mappings in the class are close to each other. In this case, by minimizing the empirical difference, the algorithm will automatically minimize the true difference as well. In order to formalize this idea, we define the notion of ``representativeness'' of a sample.


\begin{definition} {($\epsilon$-Representative Sample)} Let $\mathcal{F}$ be a class of mappings from $X$ to $\mathbb{R}^d$. A sample $S$ is $\epsilon$-representative with respect to $\mathcal{F}$, $X$ and the clustering $C^*$, if for every $f\in \mathcal{F}$ the following holds

\begin{equation}
|\Delta_{X}(C^*, C^{f}_X) - \Delta_{S}(C^*, C^{f}_X))| \leq \epsilon
\end{equation}

\end{definition}

The following theorem shows that for the TERM algorithm to work, it is sufficient to supply it with a representative sample. 

\begin{theorem} {(Sufficiency of Uniform Convergence)} Let $\mathcal{F}$ be a set of mappings from $X$ to $\mathbb{R}^d$. If $S$ is an $\frac{\epsilon}{2}$-representative sample with respect to $X$, $\mathcal{F}$ and $C^*$ then 

\begin{equation}
\Delta_X(C^*, C^{\hat{f}}_X) \leq \Delta_X(C^*, C^{f^*}_X) + \epsilon
\end{equation}

where $f^* = \argmin_{f\in \mathcal{F}} \Delta_X(C^*, C^f_X)$ and $\hat{f} = A^{TERM}(S,C^*\Big|_S)$.

\end{theorem}

\begin{proof}

Using $\frac{\epsilon}{2}$-representativeness of $S$ and the fact that $\hat{f}$ is the empirical minimizer of the loss function, we have

\begin{equation}
\Delta_X(C^*, C^{\hat{f}}_X) \leq \Delta_S(C^*, C^{\hat{f}}_X) + \frac{\epsilon}{2}
\end{equation}

\begin{equation}
\leq \Delta_S(C^*, C^{f^*}_X) + \frac{\epsilon}{2} 
\end{equation}

\begin{equation}
\leq \Delta_X(C^*, C^{f^*}_X) + \frac{\epsilon}{2}  + \frac{\epsilon}{2}
\end{equation}

\begin{equation}
\leq \Delta_X(C^*, C^{f^*}_X) + \epsilon
\end{equation}

\end{proof}

Therefore, we just need to provide an upper bound for the sample complexity of uniform convergence: ``how many instances do we need to make sure that with high probability our sample is $\epsilon$-representative?''

\subsection{Classes of Mappings with a Uniqueness Property}

In general, the solution to k-means clustering may not be unique. Therefore, the learner may end up with finding a mapping that corresponds to multiple different clusterings. This is not desirable, because in this case, the output of the learner will not be interpretable. Therefore, it is reasonable to choose the class of potential mappings in a way that it includes only the mappings under which the solution is unique. 

In order to make this idea concrete, we need to define an appropriate notion of uniqueness. We use a notion similar to the one introduced by \cite{balcan2009approximate} with a slight modification\footnote{Our notion is additive in both parameters rather than multiplicative}.

\begin{definition} {($(\eta, \epsilon)$-Uniqueness)} We say that k-means clustering for domain $X$ under mapping $f:\mathcal{X}\mapsto \mathbb{R}^d$ has a $(\eta, \epsilon)$-unique solution, if every $\eta$-optimal solution of the k-means cost is $\epsilon$-close to the optimal solution. Formally, the solution is $(\eta, \epsilon)$-unique if for every partition $P$ that satisfies

\begin{equation}
COST_{X}(f, P) < COST_{X}(f, C^{f}_X) + \eta
\end{equation}

would also satisfy

\begin{equation}
\Delta_{X}(C^{f}_X, P ) < \epsilon
\end{equation}

In the degenerate case where the optimal solution to k-means is not unique itself (and so $C^{f}_{X}$ is not well-defined), we say that the solution is not $(\eta, \epsilon)$-unique.

\end{definition}

It can be noted that the definition of $(\eta, \epsilon)$-uniqueness not only requires the optimal solution to k-means clustering to be unique, but also all the ``near-optimal'' minimizers of the k-means clustering cost should be ``similar''. This is a natural strengthening of the uniqueness condition, to guard against cases where there are $\eta_0$-optimizers of the cost function (for arbitrarily small $\eta_0$) with totally different solutions.

Now that we have a definition for uniqueness, we can define the set of mappings for $X$ under which the solution is unique. We say that a class of mappings $F$ has $(\eta, \epsilon)$-uniqueness property with respect to $X$, if every mapping in $F$ has $(\eta, \epsilon)$-uniqueness property over $X$.

Note that given an arbitrary class of mappings $F$, we can find a subset of it that satisfies $(\eta, \epsilon)$-uniqueness property over $X$. Also, as argued above, this subset is the useful subset to work with. Therefore, in the rest of the paper, we investigate learning for classes with $(\eta, \epsilon)$-uniqueness property. In the next section, we prove uniform convergence results for such classes.

\subsection{Uniform Convergence Results}

In Section 3.1, we defined the notion of $\epsilon$-representative samples. Also, we proved that if a TERM algorithm is fed with such a representative sample, it will work satisfactorily. The most technical part of the proof is then about the question ``how large should be the sample in order to make sure that with high probability it is actually a representative sample?''

In order to formalize this notion, let $\mathcal{F}$ be a set of mappings from a domain $X$ to $(0,1)^{n}$\footnote{In the analysis, for simplicity, we will assume that the set of mappings is a function to the bounded space ${(0,1)}^{n}$ wherever needed}. Define the sample complexity of uniform convergence, $m^{UC}_{\mathcal{F}}(\epsilon, \delta)$, as the minimum number $m$ such that for every fixed partition $C^*$, if $S$ is a randomly (uniformly) selected subset of $X$ with size $m$, then with probability at least $1-\delta$, for all $f\in \mathcal{F}$ we have

\begin{equation}
|\Delta_{X}(C^{*}, C^{f}_{X}) - \Delta_{S}(C^{*}, C^{f}_{X})| \leq \epsilon
\end{equation}

The technical part of this paper is devoted to provide an upper bound for this sample complexity.

\subsubsection{Preliminaries}

\begin{definition} {($\epsilon$-cover and covering number)} Let $\mathcal{F}$ be a set of mappings from $X$ to $(0,1)^{n}$. A subset $\hat{F}\subset \mathcal{F}$ is called an $\epsilon$-cover for $\mathcal{F}$ with respect to the metric $d(.,.)$ if for every $f\in \mathcal{F}$ there exists $\hat{f}\in \hat{F}$ such that $d(f, \hat{f}) \leq \epsilon$. The covering number, $\mathcal{N}(\mathcal{F}, d, \epsilon)$ is the size of the smallest $\epsilon$-cover of $\mathcal{F}$ with respect to $d$. 

\end{definition}

In the above definition, we did not specify the metric $d$. In our analysis, we are interested in the $L_1$ distance with respect to $X$, namely:



\begin{equation}
d_{L_1}^{X} (f_1, f_2) = \frac{1}{|X|}\sum_{x\in X} \|f_1(x) - f_2(x)\|_2
\end{equation}

Note that the mappings we consider are not real-valued functions, but their output is an $n$-dimensional vector. This is in contrast to the usual analysis used for learning real-valued functions. If $f_1$ and $f_2$ are real-valued, then $L_1$ distance is defined by

\begin{equation}
d_{L_1}^{X} (f_1, f_2) = \frac{1}{|X|}\sum_{x\in X} |f_1(x) - f_2(x)|
\end{equation}

We will prove sample complexity bounds for our problem based on the $L_1$-covering number of the set of mappings. However, it will be beneficial to have a bound based on some notion of capacity, similar to VC-dimension, as well. This will help in better understanding and easier analysis of sample complexity of different classes. While VC-dimension is defined for binary valued functions, we need a similar notion for functions with outputs in $\mathbb{R}^n$. For real-valued functions, we have such notion, called pseudo-dimension (\cite{pollard1984convergence}).

\begin{definition} (Pseudo-Dimension)
Let $\mathcal{F}$ be a set of functions from $X$ to $\mathbb{R}$. Let $S=\{x_1,x_2,\ldots,x_m\}$ be a subset of $X$. Then $S$ is pseudo-shattered by $\mathcal{F}$ if there are real numbers $r_1,r_2,\ldots,r_m$ such that for every $b\in \{0,1\}^m$, there is a function $f_b\in \mathcal{F}$ with $sgn(f_b(x_i)-r_i)=b_i$ for $i\in[m]$. Pseudo dimension of $\mathcal{F}$, called $Pdim(\mathcal{F})$, is the size of the largest shattered set.

\end{definition}

It can be shown (e.g., Theorem 18.4. in \cite{anthony2009neural}) that for a real-valued class $F$, if $Pdim(F)\leq q$ then $\log \mathcal{N}(F, d_{L_1}^X,\epsilon) = \mathcal{O}(q)$ where $\mathcal{O}()$ hides logarithmic factors of $\frac{1}{\epsilon}$. In the next sections, we will generalize this notion to $\mathbb{R}^n$-valued functions.

\subsubsection{Reduction to Binary Hypothesis Classes}

Let $f_1,f_2\in\mathcal{F}$ be two mappings and $\sigma$ be a permutation over $[k]$. Define the binary-valued function $h_{\sigma}^{f_1,f_2}(.)$ as follows

\begin{equation}
h_{\sigma}^{f_1,f_2}(x) = \left\{
	\begin{array}{ll}
		1 & x\in \cup_{i=1}^{k} (C^{f_1}_i \Delta C^{f_2}_{\sigma(i)}) \\
		0 & \mbox{otherwise}
	\end{array}
\right.
\end{equation}

Let $H^{\mathcal{F}}_\sigma$ be the set of all such functions with respect to $\mathcal{F}$ and $\sigma$:

\begin{equation}
H^{\mathcal{F}}_\sigma = \{ h_{\sigma}^{f_1,f_2}(.):f_1,f_2\in \mathcal{F} \}
\end{equation}

Finally, let $H^{\mathcal{F}}$ be the union of all $H^{\mathcal{F}}_\sigma$ over all choices of $\sigma$. Formally, if $\pi$ is the set of all permutations over $[k]$, then

\begin{equation}
H^{\mathcal{F}} = \cup_{\sigma \in \pi} H^{\mathcal{F}}_\sigma
\end{equation}

For a set $S$, and a binary function $h(.)$, let $h(S) = \frac{1}{|S|}\sum_{x\in S} h(x)$. We now show that a uniform convergence result with respect to $H^{\mathcal{F}}$ is sufficient to have uniform convergence for the $\Delta$-difference function. Therefore, we will be able to investigate conditions for uniform convergence of $H^{\mathcal{F}}$ rather than the $\Delta$-difference function.

\begin{theorem}
Let $X$ be a domain set, $\mathcal{F}$ be a set of mappings, and $H^{\mathcal{F}}$ be defined as above. If $S\subset X$ is such that 

\begin{equation}
\forall h \in H^{\mathcal{F}}, |h(S) - h(X)| \leq \epsilon
\end{equation}

then $S$ will be $\epsilon$-representative with respect to $\mathcal{F}$, i.e., for all $f_1,f_2 \in \mathcal{F}$ we will have
\begin{equation}
|\Delta_{X}(C^{f_1}_X, C^{f_2}_X) - \Delta_{S}(C^{f_1}_X, C^{f_2}_X)| \leq \epsilon
\end{equation}

\end{theorem}

\begin{proof}

\begin{equation}
|\Delta_S(C^{f_1}_X,C^{f_2}_X) - \Delta_X(C^{f_1}_X,C^{f_2}_X)|
\end{equation}

\begin{equation}
=\left|\left(\min_\sigma \frac{1}{|S|} \sum_{x\in S} h^{f_1,f_2}_{\sigma}\right) - \left(\min_\sigma \frac{1}{|X|} \sum_{x\in X} h^{f_1,f_2}_{\sigma}\right)\right|
\end{equation}

\begin{equation}
\leq \left| \max_\sigma \left(\frac{1}{|S|} \sum_{x\in S} h^{f_1,f_2}_{\sigma} - \frac{1}{|X|} \sum_{x\in X} h^{f_1,f_2}_{\sigma}\right)\right|
\end{equation}

\begin{equation}
\leq \left| \max_\sigma \left( h^{f_1,f_2}_{\sigma}(S) -  h^{f_1,f_2}_{\sigma}(X)  \right)\right| \leq \epsilon
\end{equation}

\end{proof}

The fact that  $H^{\mathcal{F}}$ is a class of binary-valued functions enables us to provide sample complexity bounds based on VC-dimension of this class. However, providing bounds based on VC-Dim$(H^{\mathcal{F}})$ is not sufficient, in the sense that it is not convenient to work with the class $H^{\mathcal{F}}$. Instead, it will be nice if we can prove bounds directly based on the capacity of the class of mappings, $\mathcal{F}$. In the next section, we address this issue.

\subsubsection{$L_1$-Covering Number and Uniform Convergence}

The classes introduced in the previous section, $H^{\mathcal{F}}$ and $H^{\mathcal{F}}_\sigma$, are binary hypothesis classes. Also, we have shown that proving a uniform convergence result for $H^{\mathcal{F}}$ is sufficient for our purpose. In this section, we show that a bound on the $L_1$ covering number of $\mathcal{F}$ is sufficient to prove uniform convergence for $H^{\mathcal{F}}$.

In Section 3.2, we argued that we only care about the classes that have $(\eta, \epsilon)$-uniqueness property. In the rest of this section, assume that $\mathcal{F}$ is a class of mappings from $X$ to $(0,1)^n$ that satisfies $(\eta, \epsilon)$-uniqueness property.

\begin{lemma}
Let $f_1,f_2\in \mathcal{F}$. If $d_{L_1}(f_1,f_2) < \frac{\eta}{12}$ then $\Delta_X(f_1, f_2) < 2\epsilon$
\end{lemma}

We leave the proof of this lemma for the appendix, and present the next lemma.

\begin{lemma}
Let $H^{\mathcal{F}}$ be defined as in the previous section. Then,

\begin{equation}
\mathcal{N}(H^{\mathcal{F}}, d_{L_1}^{X}, 2\epsilon) \leq k!\mathcal{N}(\mathcal{F}, d_{L_1}^{X}, \frac{\eta}{12})
\end{equation}
\end{lemma}

\begin{proof}
Let $\hat{\mathcal{F}}$ be the $\frac{\eta}{12}$-cover corresponding to the covering number $\mathcal{N}(\mathcal{F}, d_{L_1}^{X}, \frac{\eta}{12})$. Based on the previous lemma, $H^{\hat{\mathcal{F}}}_\sigma$ is a $2\epsilon$-cover for $H^{{\mathcal{F}}}_\sigma$. But we have only $k!$ permutations of $[k]$, therefore, the covering number for $H^{\hat{\mathcal{F}}}$ is at most $k!$ times larger than $H^{\hat{\mathcal{F}}}_\sigma$. This proves the result.
\end{proof}

Basically, this means that if we have a small $L_1$ covering number for the mappings, we will have the uniform convergence result we were looking for. The following theorem proves this result.

\begin{theorem} Let $\mathcal{F}$ be a set of mappings with $(\eta, \epsilon)$-uniqueness property. Then there for some constant $\alpha$ we have

\begin{equation}
m^{UC}_{\mathcal{F}}(\epsilon, \delta) \leq O(\frac{ \log k! + \log \mathcal{N}(\mathcal{F}, d_{L_1}^{X}, \frac{\eta}{\alpha})+\log(\frac{1}{\delta})}{\epsilon^2})
\end{equation}

\end{theorem}

\begin{proof}
Following the previous lemma, if we have a small $L_1$-covering number for $\mathcal{F}$, we will also have a small covering number for $H^{\mathcal{F}}$ as well. But based on standard uniform convergence theory, if a hypothesis class has small covering number, then it has uniform convergence property. More precisely, (e.g., Theorem 17.1 in \cite{anthony2009neural}) we have:
\begin{equation}
m^{UC}_{H^{\mathcal{F}}}(\epsilon_0, \delta) \leq O(\frac{ \log \mathcal{N}(H^{\mathcal{F}}, d_{L_1}^{X}, \frac{\epsilon_0}{16}) +\log(\frac{1}{\delta})}{\epsilon_0^2})
\end{equation}

Applying Lemma 2 to the above proves the result.
\end{proof}

%
%

\subsubsection{Bounding $L_1$-Covering Number }

In the previous section, we proved if the $L_1$ covering number of the class of mappings is bounded, then we will have uniform convergence. However, it is desirable to have a bound with respect to a combinatorial dimension of the class (rather than the covering number). Therefore, we will generalize the notion of pseudo-dimension for the class of mappings that take value in $\mathbb{R}^n$.

Let $\mathcal{F}$ be a set of mappings form $X$ to $\mathbb{R}^n$. For every mapping $f\in \mathcal{F}$, define real-valued functions $f_1,\ldots,f_n$ such that $f(x)= (f_1(x),\ldots,f_n(x))$. Now let $F_i = \{ f_i: f\in F\}$. This means that $F_1, F_2,\ldots,F_n$ are classes of real-valued functions. Now we define pseudo-dimension of $\mathcal{F}$ as follow.

\begin{equation}
Pdim(\mathcal{F}) = n \max_{i\in[n]} Pdim(F_i)
\end{equation}

\begin{proposition} Let $\mathcal{F}$ be a set of mappings form $X$ to $\mathbb{R}^n$. If $Pdim(F)\leq q$ then $\log \mathcal{N}(F, d_{L_1}^X,\epsilon) = \mathcal{O}(q)$ where $\mathcal{O}()$ hides logarithmic factors.

\end{proposition}

\begin{proof} The result follows from the corresponding result for bounding covering number of real-valued functions based on pseudo-dimension mentioned in the preliminaries section. The reason is that we can create a cover by composition of the $\frac{\epsilon}{n}$-covers of all $F_i$. However, this will at most introduce a factor of $n$ in the logarithm of the covering number.
\end{proof}

Therefore, we can rewrite the result of the previous section in terms of pseudo-dimension.

\begin{theorem} Let $\mathcal{F}$ be a class of mappings with $(\eta, \epsilon)$-uniqueness property. Then 

\begin{equation}
m^{UC}_{\mathcal{F}}(\epsilon, \delta) \leq \mathcal{O}(\frac{ k + Pdim(\mathcal{F})+\log(\frac{1}{\delta})}{\epsilon^2})
\end{equation}

where $\mathcal{O}()$ hides logarithmic factors of $k$ and $\frac{1}{\eta}$.

\end{theorem}

\subsection{Sample Complexity of PAC-SRLK}

In Section 3.1, we showed that uniform convergence is sufficient for a TERM algorithm to work. Also, in the previous section, we proved a bound for the sample complexity of uniform convergence. The following theorem, which is the main technical result of this paper, combines these two and provides a sample complexity upper bound for PAC-SRLK framework.

\begin{theorem}

Let $\mathcal{F}$ be a class of $(\eta, \epsilon)$-unique mappings. Then the sample complexity of learning representation for $k$-means clustering with respect to $\mathcal{F}$ is upper bounded by

\begin{equation}
m_{\mathcal{F}}(\epsilon, \delta) \leq \mathcal{O}(\frac{ k + Pdim(\mathcal{F})+\log(\frac{1}{\delta})}{\epsilon^2})
\end{equation}

where $\mathcal{O}$ hides logarithmic factors of $k$ and $\frac{1}{\eta}$.

\end{theorem}

The proof is done by combining Theorems 1 and 4. 


The following result shows an upper bound for the sample complexity of learning linear mappings (or equivalently, Mahalanobis metrics).

\begin{corollary} Let $\mathcal{F}$ be a set of $(\eta,\epsilon)$-unique \emph{linear} mappings from $\mathbb{R}^{d_1}$ to $\mathbb{R}^{d_2}$. Then we have

\begin{equation}
m_{\mathcal{F}}(\epsilon, \delta) \leq \mathcal{O}(\frac{ k + d_1d_2+\log(\frac{1}{\delta})}{\epsilon^2})
\end{equation}

\end{corollary}

\begin{proof}
It is a standard result that the pseudo-dimension of a vector space of real-valued functions is just the dimensionality of the space (in our case $d_1$) (e.g., Theorem 11.4 in \cite{anthony2009neural}). Also, based on our definition of $Pdim$ for $\mathbb{R}^{d_2}$-valued functions, it should scale by a factor of $d_2$.
\end{proof}

\section{CONCLUSIONS AND OPEN PROBLEMS}
\label{sec:conc}

In this paper we provided a formal statistical framework for learning the representation (i.e., a mapping) for k-means clustering based on supervised feedback. The learner, unaware of the target clustering of the domain, is given a clustering of a sample set. The learner's task is then finding a mapping function $\hat{f}$ (among a class of mappings) under which the result of k-means clustering of the domain is as close as possible to the true clustering. This framework was called PAC-SRLK.

A notion of $\epsilon$-representativeness was introduced, and it was proved that any ERM-type algorithm that has access to such a sample will work satisfactorily. Finally, a technical uniform convergence result was proved to make sure that a large enough sample is (with high probability) $\epsilon$-representative. This was used to prove an upper bound for the sample complexity of PAC-SRLK based on covering numbers of the set of mappings. Furthermore, a notion of pseudo-dimension for the class of mappings was defined, and the sample complexity was upper bounded based on it. 

Note that in the analysis, the notion of $(\eta, \epsilon)$-uniqueness (similar to that of \cite{balcan2009approximate}) was used and it was argued that it is reasonable to require the learner to output a mapping under which the solution is ``unique'' (because otherwise the output of k-means clustering would not be interpretable). Therefore, in the analysis, we assumed that the class of potential mappings has the $(\eta, \epsilon)$-uniqueness property.

It can be noted that we did not analyze the computational complexity of algorithms for PAC-SRLK framework. We leave this analysis to the future work. We just note that a similar notion of uniqueness proposed by \cite{balcan2009approximate} resulted in being able to efficiently solve the k-means clustering algorithm.

One other observation is that representation learning can be regarded as a special case of metric learning; because for every mapping, we can define a distance function that computes the distance in the mapped space. In this light, we can make the problem more general by making the learner find a distance function rather than a mapping. This is more challenging to analyze, because we do not even know a generalization bound for center-based clustering under general distance functions. An open question will be providing such general results.

\subsubsection*{Acknowledgments}

\section{APPENDIX}

{Proof of Lemma 1.} Let $\mathcal{F}:X\mapsto (0,1)^n$ be a set of mappings that have $(\eta, \epsilon)$-uniqueness property. Let $f_1,f_2\in \mathcal{F}$ and $d_{L_1}(f_1,f_2) < \frac{\eta}{12}$. We need to prove that $\Delta_X(f_1, f_2) < 2\epsilon$. In order to prove this, note that due to triangular inequality, we have

\begin{multline}
\Delta_X(f_1, f_2) = \Delta_X(C^{f_1}(\mu^{f_1}), C^{f_2}(\mu^{f_2})) \\
\leq \Delta_X(C^{f_1}(\mu^{f_1}), C^{f_1}(\mu^{f_2})) +\\ \Delta_X(C^{f_1}(\mu^{f_2}), C^{f_2}(\mu^{f_2}))
\end{multline}

Therefore, it will be sufficient to show that each of the $\Delta$-terms above is smaller than $\epsilon$. We start by proving a useful lemma.

\begin{lemma}
Let $f_1,f_2\in \mathcal{F}$ and $d_{L_1}(f_1,f_2) < \frac{\eta}{6}$. Let $\mu$ be an arbitrary set of $k$ centers in $(0,1)^n$. Then

$$|COST_X(f_1, \mu) - COST_X(f_2, \mu)| < \frac{\eta}{2}$$
\end{lemma}

\begin{proof}

\begin{multline}
|COST_X(f_1, \mu) - COST_X(f_2, \mu)|\\
= \Bigg|\left(\frac{1}{|X|} \sum_{x\in X} \min_{\mu_j\in\mu}\| f_1(x)-\mu_j  \|^2\right)  \\
 - \left(\frac{1}{|X|}\sum_{x\in X} \min_{\mu_j\in\mu}\| f_2(x)-\mu_j  \|^2 \right)\Bigg| 
\end{multline}

\begin{equation}
 \leq \frac{1}{|X|} \sum_{x\in X} \max_{\mu_j\in\mu} \Big|  \| f_1(x)-\mu_j  \|^2- \| f_2(x)-\mu_j  \|^2  \Big|
\end{equation}

\begin{equation}
 = \frac{1}{|X|} \sum_{x\in X} \max_{\mu_j\in\mu} \Big|  
 \|f_1(x)\|^2 - \|f_2(x)\|^2 - 2 <\mu_j, f_1-f_2>
   \Big|
\end{equation}

\begin{equation}
 = \frac{1}{|X|} \sum_{x\in X} \max_{\mu_j\in\mu} \Big|  
 <f_1 - f_2, f_1 + f_2 - 2 \mu_j>
   \Big|
\end{equation}

\begin{equation}
 \leq \frac{3}{|X|} \sum_{x\in X} \| f_1 - f_2\| \leq \frac{3\eta}{6} \leq \frac{\eta}{2}
\end{equation}

\end{proof}

Now we are ready to prove that the first $\Delta$-term is smaller than $\epsilon$, i.e.,  $\Delta_X(C^{f_1}(\mu^{f_1}), C^{f_1}(\mu^{f_2})) < \epsilon$. But to do so, we only need to show that $COST_X(f_1, \mu^{f_2}) - COST_X(f_1, \mu^{f_1}) < \eta$; because in that case, due to $(\eta,\epsilon)$-uniqueness property of $f_1$, the result will follow. Now, using Lemma 3, we have

\begin{equation}
COST_X(f_1, \mu^{f_2}) - COST_X(f_1, \mu^{f_1})
\end{equation}

\begin{equation}
\leq \left( COST_X(f_2, \mu^{f_2}) + \frac{\eta}{2}\right) - COST_X(f_1, \mu^{f_1})
\end{equation}

\begin{equation}
=  \min_{\mu}(COST_X(f_2, \mu)) - \min_{\mu}(COST_X(f_1, \mu)) + \frac{\eta}{2}
\end{equation}

\begin{equation}
\leq  \max_{\mu}\left(COST_X(f_2, \mu) - COST_X(f_1, \mu)\right) + \frac{\eta}{2}
\end{equation}

\begin{equation}
\leq \frac{\eta}{2}  + \frac{\eta}{2} \leq \eta
\end{equation}

where in the first and the last line we used Lemma 3.

Finally, we need to prove the second $\Delta$-inequality, i.e., $\Delta_X(C^{f_1}(\mu^{f_2}), C^{f_2}(\mu^{f_2})) \leq \epsilon$. Assume contrary. But based on $(\eta,\epsilon)$-uniqueness property of $f_2$, we conclude that $COST_X(f_2, C^{f_1}(\mu^{f_2})) - COST_X(f_2, C^{f_2}(\mu^{f_2})) \geq \eta$. In the following, we prove that this cannot be true, and hence a contradiction.

Let $m_x = \argmin_{\mu_0\in \mu^{f_2}} \|f_1(x)-\mu_0 \|^2$. Then, based on the boundedness of $f_1(x)$,$f_2(x)$ and  we have:

\begin{equation}
COST_X(f_2, C^{f_1}(\mu^{f_2})) - COST_X(f_2, C^{f_2}(\mu^{f_2}))
\end{equation}

\begin{equation}
=\left( \frac{1}{|X|}\sum_{x\in X} 
\| f_2(x) - m_x \|^2
\right)  - COST_X(f_2, \mu_2)
\end{equation}

\begin{multline}
=\left( \frac{1}{|X|}\sum_{x\in X} 
\| f_2(x) - f_1(x) + f_1(x) - m_x \|^2 \right)\\
 - COST_X(f_2, \mu_2)
\end{multline}

\begin{multline}
= \frac{1}{|X|}\sum_{x\in X} \| f_2(x) - f_1(x) \|^2 \\
+ \frac{1}{|X|}\sum_{x\in X} \| f_1(x) - m_x \|^2 \\
+ \frac{1}{|X|}\sum_{x\in X} 2<f_2(x) - f_1(x),f_1(x)- m_x>\\
 - COST_X(f_2, \mu_2)
\end{multline}

\begin{multline}
\leq \frac{2}{|X|}\sum_{x\in X} \| f_2(x) - f_1(x) \| \\
+ COST_X(f_1, \mu_1) \\
+ \frac{4}{|X|}\sum_{x\in X} \|f_2(x) - f_1(x)\|\\
 - COST_X(f_2, \mu_2)
\end{multline}

\begin{multline}
\leq \frac{6}{|X|}\sum_{x\in X} \| f_2(x) - f_1(x) \| \\
+ \left( COST_X(f_1, \mu_1) - COST_X(f_2, \mu_2) \right)\\
\end{multline}

\begin{equation}
\leq \frac{6\eta}{12} + \frac{\eta}{2} \leq \eta
\end{equation}

\bibliographystyle{apalike}
\bibliography{uai}

\end{document}